\documentclass[11pt, letterpaper]{article}
\usepackage{fullpage}
\usepackage[T1]{fontenc}

\usepackage[colorlinks=true,pdfpagemode=none,linkcolor=blue,citecolor=blue]{hyperref}

\usepackage[round]{natbib}  
\usepackage{caption} 

\usepackage{algorithm, algorithmic}
\usepackage{xcolor}
\usepackage{subfigure}
\usepackage{amsmath, amsthm, amssymb}
\usepackage{multicol}
\usepackage{enumitem}

\usepackage{eqparbox}
\renewcommand{\algorithmiccomment}[1]{\eqparbox{COMMENT}{// #1}}

\newcommand\Tr{T_{\textnormal{round}}}
\newcommand\Tb{T_{\textnormal{batch}}}

\newcommand\D{\mathcal{D}}
\newcommand\E{\mathbb{E}}
\newcommand\R{\mathbb{R}}
\newcommand\Reg{\mathcal{R}}
\newcommand\reg{\mathcal{R}}
\newcommand\hatreg{\widehat{\mathcal{R}}}
\newcommand\hatf{\widehat{f}}
\newcommand\tilf{\widetilde{f}}

\newcommand\x{\boldsymbol{x}}

\newcommand\cc{\boldsymbol{c}}
\newcommand\hatx{\widehat{\boldsymbol{x}}}

\newcommand\y{\boldsymbol{y}}
\newcommand\g{\boldsymbol{g}}
\newcommand\gtil{\boldsymbol{\widetilde{g}}}
\newcommand\tilg{\boldsymbol{\widetilde{g}}}
\newcommand\s{\boldsymbol{s}}
\newcommand\stil{\boldsymbol{\widetilde{s}}}
\newcommand\tils{\boldsymbol{\widetilde{s}}}
\newcommand\vv{\boldsymbol{v}}

\newcommand\uu{\boldsymbol{u}}
\newcommand\xtil{\boldsymbol{\widetilde{x}}}

\newcommand\Dist{\mathcal{P}}

\newcommand\Z{\boldsymbol{Z}}
\newcommand\mI{\boldsymbol{I}}

\newcommand\ones{\boldsymbol{1}}

\newcommand\calA{\mathcal{A}}

\renewcommand\epsilon{\varepsilon}
\newcommand\eps{\varepsilon}
\newcommand\epscg{\varepsilon_{\textnormal{cg}}}

\newcommand\ndel{\zeta} 

\numberwithin{equation}{section}
\numberwithin{figure}{section}
\theoremstyle{plain}
	\newtheorem{theorem}{Theorem}[section]
	\newtheorem{thm}[theorem]{Theorem}

	\newtheorem{lem}[theorem]{Lemma}

 	\newtheorem{cor}[theorem]{Corollary}

\theoremstyle{definition}
	\newtheorem{definition}[theorem]{Definition}
	
	\newtheorem*{remark*}{Remark}

\begin{document}

\clubpenalty=10000
\widowpenalty = 10000

\title{Projection-Free Bandit Optimization with Privacy Guarantees}
\author {
Alina Ene\thanks{Department of Computer Science, Boston University, \tt{aene@bu.edu}}
        \and
        Huy L. Nguy\~{\^{e}}n\thanks{Khoury College of Computer and Information Science, Northeastern University, \tt{hu.nguyen@northeastern.edu}} 
        \and
        Adrian Vladu\thanks{CNRS \& IRIF, Universit\'{e} de Paris, \tt{vladu@irif.fr}}
}

\date{}
\maketitle

\begin{abstract}
We design differentially private algorithms for the bandit convex optimization problem in the projection-free setting. 
This setting is important whenever the decision set has a complex geometry, and access to it is done efficiently only through a linear optimization oracle, hence Euclidean projections are unavailable (e.g. matroid polytope, submodular base polytope).
This is the first differentially-private algorithm for projection-free bandit optimization, and in fact our bound of $\widetilde{O}(T^{3/4})$ matches the best known non-private projection-free algorithm (Garber-Kretzu, AISTATS `20) and the best known private algorithm, even for the weaker setting when projections are available (Smith-Thakurta, NeurIPS `13).
\end{abstract}

\section{Introduction}
Online learning is a fundamental optimization paradigm employed in settings where one needs to make decisions in an uncertain environment. Such methods are essential for a range of practical applications: ad-serving~\citep{mcmahan2013ad}, dynamic pricing~\citep{LobelLV17,mao2018contextual},
or recommender systems~\citep{abernethy2007online} are only a few examples.
These techniques are highly dependent on access to certain user data, such as search history, list of contacts, etc. which may expose sensitive information about a particular person. 

As these tools become ubiquitous on the internet, one can witness a surge in the
collection of user data at massive scales. This is a tremendous problem, since 
by obtaining information about the behavior of algorithms run on these data,
adversarial entities may learn potentially sensitive information; this could then be traced to a particular user, even if the users were anonymized to begin with~\citep{dwork2014algorithmic}.

To mitigate the threat of diminishing user privacy, one can leverage the power of differential privacy~\citep{dwork2006calibrating}, a notion of privacy which ensures that the 
output of an algorithm is not sensitive to the presence of a particular user's data. Therefore, based on this output, one can not determine whether a user presents one or more given attributes.

Differentially private learning algorithms have been studied in several settings, and a large number of recent works addressed the challenge of designing general optimization primitives with privacy guarantees~\citep{jain2012differentially, agarwal2017price, bassily2014private, bassily2014private2, abadi2016deep, wang2017differentially, iyengar2019towards}.
In this paper, we further advance this line of research by offering differentially private algorithms for a very general task -- the bandit convex optimization problem in the case where the space of decisions that the learning algorithm can make exhibits complex geometry.

Bandit convex optimization is an extremely general framework for online learning, which is motivated by the natural setting where, after making a decision the algorithm only learns the loss associated with its action, and nothing about other possible decisions it could have made (as opposed to the weaker \textit{full information} model where losses associated to all the possible decisions are revealed). Algorithms for this problem are highly dependent on the geometric properties of the space of decisions -- and their performance usually depends on the ability to perform certain projections onto this space~\citep{ben2001lectures,jaggi2013revisiting}. For large scale problems, this requirement may be prohibitive, as decisions may have to satisfy certain constraints (the set of recommendations must be diverse enough, or the set of ads to be displayed satisfy a given budget). Projection-free methods overcome this issue by exploiting the fact that some canonical decision sets often encountered in applications (matroid polytope, submodular base polytope, flow polytope, convex relaxations of low-rank matrices) have efficient linear optimization oracles. One can therefore use these efficient oracles in conjunction with the projection-free method to obtain algorithms that can be deployed for real-world applications.

In this work we bridge the requirements of privacy and efficiency for online learning, building on the works of~\citep{garber2019improved, garber2013playing} to obtain the first differentially private algorithm for projection-free bandit optimization.
To do so we leverage a generic framework for online convex optimization in the presence of noise, which we then adapt to our specific setting in a modular fashion.

\medskip
{\bf Our Contributions.}
We give the first differentially private algorithm for the bandit convex optimization problem in the projection-free setting (we defer the definition of $(\epsilon,\delta)$-privacy to Definition~\ref{def:dp} and the problem statement to Section~\ref{sec:prelim}). We summarize the regret guarantees of our algorithm in the following theorem and compare it with the state of the art guarantees in the private and non-private settings. Our main focus is on the dependency on the dimension $n$ of the ambient space, the number $T$ of iterations, and the privacy budget $\epsilon$. For ease of comparison, we use the $\widetilde{O}$ notation to hide poly-logarithmic factors in $n$ and $T$, as well as parameters such as the Lipschitz constant of the loss functions. The precise guarantees can be found in Lemma~\ref{lem:reglap} (for $(\epsilon,0)$-privacy) and Lemma~\ref{lem:reggauss} (for $(\epsilon,\delta)$-privacy). 

\begin{thm}\label{thm:maininfo}  
Let $\D\subseteq \R^n$ be a convex domain for which we have access to a linear optimization oracle.
Assume that for every $t \geq 1$, $f_t$ is convex and $L$-Lipschitz. Furthermore suppose that $\max_{\x,\y\in\D} \|x-y\| \leq D$. Then there exists an algorithm $\textsc{PrivateBandit}$ (Algorithm~\ref{alg:bandit}) which performs projection-free convex optimization in the bandit setting 
such that one of the following two properties holds:

\begin{itemize}[leftmargin=*,noitemsep,topsep=0pt]
\item the algorithm is $(\epsilon, 0)$-differentially private and, assuming $L=O(1)$ and $D=O(1)$, achieves an expected regret of  $$\Reg_T = \widetilde{O}\bigg(\frac{T^{3/4}n^{3/2}} {\epsilon} \bigg) \,.$$
\item the algorithm is $(\epsilon, \delta)$-differentially private and, assuming $L=O(1)$ and $D=O(1)$, achieves an expected regret of  $$\Reg_T = \widetilde{O}\bigg(\frac{(T^{3/4}n^{1/2}+T^{1/2}n) \log^{O(1)}(1/\delta) }{\epsilon} \bigg)\,,$$
whenever $\delta = 1/(n+T)^{O(1)}$.
\end{itemize}
\end{thm}

In the non-private setting, the state of the art regret guarantee for projection-free bandit optimization is $\widetilde{O}(n^{1/2} T^{3/4})$ due to Garber and Kretzu~\citep{garber2019improved}\footnote{Their paper allows for a trade-off among parameters. This bound is optimized for the case when $T\gg n$.}. The regret guarantee of our algorithm matches the guarantee of ~\citep{garber2019improved} up to a $n/\epsilon$ factor in the $(\epsilon, 0)$ regime, and a $1/\epsilon$ factor in the $(\epsilon, \delta)$-regime, whenever $T \geq n^2$. 

Prior works in the private setting require projections to be available. The state of the art guarantees for private bandit optimization with projections are achieved by the work of Smith and Thakurta~\citep{thakurta2013nearly}. Smith and Thakurta focus on $(\epsilon,0)$-privacy and obtain a regret bound of $\widetilde{O}(n T^{3/4}/\epsilon)$. A variant of their algorithm can be used for $(\epsilon,\delta)$-privacy and obtains a regret bound of $\widetilde{O}(\sqrt{n} T^{3/4}/\epsilon)$. Our algorithm's guarantee matches the best guarantee with projections for $(\epsilon,\delta)$-privacy and is worse by a $\sqrt{n}$ factor for $(\epsilon,0)$-privacy. We leave it as an interesting open problem to improve the bound for $(\epsilon, 0)$-privacy to match the one using projections.

\medskip
{\bf Our Techniques.} In the process of obtaining our main result, we develop the common abstraction of noisy mirror descent to capture both online bandit optimization and private optimization (the \textsc{NoisyOCO} framework). This allows us to analyze the impact of the noise introduced to protect privacy on the regret of the online optimization. Once the framework is set up, one only needs to analyze the noise level to ensure the appropriate privacy guarantee and one immediately obtains the corresponding regret bound. However, analyzing the noise is in itself a non-trivial challenge. In the case of $(\epsilon,\delta)$-privacy, we give a strong concentration bound allowing us to match the privacy-regret trade-off achieved with projections (see Lemmas~\ref{lem:randvec} and \ref{lem:privgauss}).

In this case, the straightforward approach leads to worse bounds and one of our main contributions is to improve the bound under $(\epsilon,\delta)$-differential privacy by using concentration bounds and ignoring the tail. By contrast, in $(\epsilon,0)$-differential privacy, one cannot ignore what happens in the tail of the distribution and understanding the algorithm in that regime seems difficult.

We believe our framework is general and it facilitates further progress in differentially private optimization. We demonstrate our framework by instantiating it with the Laplace mechanism (to obtain an $(\epsilon, 0)$-private algorithm) and with the Gaussian mechanism (to obtain an $(\epsilon, \delta)$-private algorithm). It would be interesting to apply our framework to other notions of differential privacy, such as  concentrated differential privacy~\citep{dwork2016concentrated,bun2016concentrated} and Renyi differential privacy~\citep{mironov2017renyi}.

While we resort to established techniques from differential privacy (Gaussian and Laplacian mechanisms, tree based aggregation), properly integrating them with optimization methods does require some degree of care. 

For example, our $(\epsilon, \delta)$-privacy bound is derived using a matrix concentration inequality which crucially relies on a randomized smoothing technique used for obtaining gradient estimates. This is a key ingredient to obtaining the correct $\widetilde{O}(n^{1/2})$ dependence in dimension in the $(\epsilon, \delta)$ regime.

The final algorithm is simple but effective, matching the best known bound with projection in $(\epsilon, \delta)$-DP. We see it as a proof of concept for a general approach to deriving differentially private optimization methods. Previous results in this area can be recovered by following our approach: inject the maximum amount of noise as to not change the convergence rate asymptotically, then analyze the privacy loss. This is very simple, but it paves the way for further development of practical differentially private algorithms, without requiring major changes in their implementation -- simply replace certain components of the algorithm with a black box implementation of the required differentially private mechanism.

\medskip
{\bf Other Related Work.}
The theory of online convex optimization is truly extensive, and has seen a lot of developments in the recent years. Here, we will refer to the relevant works that involve projection-free and/or differentially private online learning algorithms. 
The class of projection-free online learning algorithms was initiated by the work of~\citep{hazan2012projection} in the context of online convex optimization, where full gradients are revealed after making a decision.  This was further extended to multiple regimes~\citep{garber2013playing, garber2013linearly, garber2015faster} including the bandit setting~\citep{chen2018projection,garber2019improved}.

As discussed above, Smith and Thakurta~\citep{thakurta2013nearly} achieve the state of the art regret guarantee for $(\epsilon,0)$-private online bandit optimization when projections are available. For general Lipschitz functions, their regret is $\widetilde{O}(n T^{3/4}/\epsilon)$. In the specific case where the adversary is oblivious and the loss functions are strongly-convex, they improve this to $\widetilde{O}(n T^{2/3}/\epsilon)$.

In a different line of work,~\citep{agarwal2017price} obtained improved bounds for the case where losses are linear functions and for the multi-armed bandit problem (a generalization of the learning with experts framework), with regret $\widetilde{O}(T^{2/3}/\epsilon)$ and $\widetilde{O}(nT^{2/3}/\epsilon)$ respectively. These results, however, concern only a restricted class of bandit optimization problems, so for the general bandit convex optimization problem the result of~\citep{thakurta2013nearly} still stands as the best one. 

In fact, even in the non-private setting  improving the regret of the bandit convex optimization problem from $\widetilde{O}(T^{3/4})$ to $\widetilde{O}(T^{2/3})$~\citep{awerbuch2004adaptive,dani2006robbing} or below~\citep{dani2008price,abernethy2009competing,bubeck2017kernel} requires stronger access to the set of actions than just projections (such as via a self-concordant barrier), and involves performing expensive computations. Indeed,~\citep{bubeck2017kernel} is the first to achieve both optimal regret $\widetilde{O}(T^{1/2})$ and polynomial running time per iteration.

\section{Preliminaries}\label{sec:prelim}
{\bf Bandit Convex Optimization.}
In the bandit convex optimization problem, an algorithm iteratively selects actions $\x_t$ (using a possibly randomized strategy) from a convex set $\D \subseteq \R^n$. After selecting an action, the loss caused by this choice $f_t(\x_t)$  is revealed, where $f_t : \D \rightarrow \R$ is a convex function unknown to the algorithm.

After performing this for $T$ iterations, the algorithm compares its  total loss $\sum_{t=1}^T  f_t(\x_t)$ to the smallest loss it could have incurred by choosing a fixed strategy throughout all the iterations $\min_{\x \in \D} \sum_{t=1}^T f_t(\x)$. The difference between these two losses is called \textit{regret}:
\[
\Reg_T = \sum_{t=1}^T f_t(\x_t) - \min_{\x \in \D} \sum_{t=1}^T f_t(\x)
\]
and the goal is to minimize its expectation over the randomized choices of the algorithm.

\medskip
{\bf Differential Privacy.}
Differential privacy~\citep{dwork2006calibrating} is a rigorous framework used to control the amount of information leaked when performing computation on a private data set. In our framework, we seek algorithms which ensure that the amount of information an adversary can learn about a particular loss function $f_t$ is minimal, i.e. it is almost independent on whether $f_t$ appears or not in the sequence of loss functions occurring throughout the execution of the algorithm.
For completeness, we define differential privacy in the context of loss functions encountered in the bandit convex optimization problem.
\begin{definition}[$(\epsilon,\delta)$-differential privacy]\label{def:dp}
A randomized online learning algorithm $\calA$ on the action set $\D$ is $(\epsilon,\delta)$-differentially private if for any two sequences of loss functions $F = (f_1, \dots, f_T)$ and $F' = (f_1', \dots, f_T')$ differing in at most one element, for all $S \subseteq \D^T$ it holds that
\[
\Pr[\calA(F) \in S] \leq e^\eps \Pr[ \calA(F') \in S ] + \delta\,.
\]
\end{definition}
One obstacle that may occur in the context of bandit optimization is that changing a single loss function may alter the set of actions returned in the future by the algorithm.

\medskip
{\bf The Projection-Free Setting.}
While classical online optimization methods have a long history of developments, these rely in general on the ability to perform projections onto the feasible set $\D$ of actions. For example, one may want to choose actions that correspond to points inside a matroid polytope, or other complicated domains. In such situations, it is computationally infeasible to perform projections onto $\D$, and designing algorithms where all the actions lie inside this domain becomes a challenging task.
In the case of online optimization, this issue is mitigated by \textit{projection-free methods}~\citep{jaggi2013revisiting,garber2015faster,dudik2012lifted,shalev2011large}, where the complexity of the high complexity of the description of $\D$ is balanced by the existence of a linear optimization oracle over this domain. Among these, the conditional gradient method (also known as Frank-Wolfe)~\citep{bubeck2015convex} is the best known one.

In our setting, we treat the case where, although $\D$ may be very complicated, we have access to such an oracle which given any direction $\vv\in\R^n$ returns
$\arg\min_{\x \in \D} \langle \vv, \x \rangle$. Such oracles are easily available for interesting domains such as the matroid polytope, or the submodular base polytope.

\medskip
{\bf Parameters and Assumptions.}
We write vectors and matrices in bold. We use $\langle \x, \y \rangle$ to represent inner products, and $\|\x\|$ to represent the $\ell_2$ norm of a vector $\|\x\| = \left(\sum_i x_i\right)^{1/2}$. When considering other norms than $\ell_2$ we explicitly specify them $\|\x\|_p = \left(\sum_i x_i^p \right)^{1/p}$.
We let $B_p^n$ be the $n$ dimensional unit $\ell_p$ ball and $S_p^n$ the boundary of $B_p^n$ i.e. the $n$ dimensional unit $\ell_p$ sphere.
We consider optimizing over a convex domain $\D \subseteq \R^n$, for which we have access to a linear optimization oracle. We define the diameter of the domain as $D = \max_{\x,\y\in\D} \|x-y\|$.
We say that a function $f:\D\rightarrow \R$ is $L$-Lipschitz iff $\vert f(\x) - f(\y) \vert \leq L \|\x-\y\|$ for all $\x,\y\in\D$ and that a differentiable function $f$ is $\beta$-smooth iff $\|\nabla f(\x) - \nabla f(\y)\| \leq \beta\|\x-\y\|$ for all $\x,\y \in \D$. We say that $f$ is $\alpha$-strongly convex iff  $\|\nabla f(\x) - \nabla f(\y)\| \geq \alpha\|\x-\y\|$. In our setting, all functions $f_t$ satisfy the standard assumption of being $L$-Lipschitz.

Just like in prior works~\citep{thakurta2013nearly,agarwal2017price}, we further assume that the number of iterations $T$ we run the algorithm for is known ahead of time. This assumption can be eliminated via a standard reduction using the doubling trick (see~\citep{auer1995gambling} and~\citep{chen2018projection}), which invokes the base algorithm repeatedly by doubling the horizon $T$ at each invocation, at the expense of adding an extra $O(\log T)$ factor in the privacy loss.

For simplicity further assume that all $f_t$'s are defined within a region that slightly exceeds the boundary of $\D$. This assumption is required, since one of the techniques employed here requires having $f_t$  defined over $\D \oplus \ndel B_2^n$ for a small scalar $\ndel$. This assumption can be removed via a simple scaling trick, whenever $\D$ contains an $\ell_2$ ball centered at the origin (similarly to~\citep{garber2019improved}); we explain how to do so in Appendix~\ref{sec:domain}.

Finally, in order to be able to appropriately privatize the losses $f_t(\x_t)$ we require bounding their magnitude. To do so we assume that each $f_t$ achieves $0$ loss at some point within $\D$, which via the Lipschitz condition and the diameter bound automatically implies that $\vert f_t(\x_t) \vert \leq LD$ for all $t$. Other related works~\citep{flaxman2004online, agarwal2010optimal, thakurta2013nearly} simply use a fixed upper bound $\vert f_t(\x_t) \vert \leq B$ for some fixed parameter $B$, but we prefer this new convention to reduce the number of parameters to control, since we are focused mainly in the regret dependence in $T$, $n$ and $\epsilon$.

\medskip
{\bf Mirror Maps and the Fenchel Conjugate.}
In general, convex optimization implicitly relies on the existence of a \textit{mirror map} $\omega : \D \rightarrow \R$ with desirable properties (see~\citep{ben2001lectures} for an extensive treatment of these objects). This is used in order to properly interface iterates and gradient updates, since in Banach spaces these are of different types. In our specific case we use $\omega(\x) = \frac{1}{2} \|\x\|_2^2$, although other choices can be used depending on the geometry of $\D$. We define the Fenchel conjugate of $\omega$ as $\omega^* : \R^n \rightarrow \R$ such that
\begin{equation}\label{eq:fenchelyoung}
\omega^*(\y) = \max_{\x \in \D} \langle \y, \x \rangle - \omega(\x)\,.
\end{equation}
Furthermore, one has that whenever $\omega$ is strongly convex, $\omega^*$ is smooth and differentiable~\citep{nesterov2005smooth}, and its gradient satisfies
\begin{equation}
\nabla \omega^*(\y) = \arg\max_{\x \in \D} \langle \y, \x \rangle - \omega(\x)\,.
\end{equation}

\medskip
{\bf Smoothing.} We use the randomized smoothing technique from~\citep{flaxman2004online} in order to smoothen the loss functions $f_t$. This technique is crucial to obtain gradient estimators despite having only value access.
\begin{lem}[\citep{flaxman2004online}]\label{lem:smooth}
Let $f:\mathbb{R}^{n}\rightarrow\mathbb{R}$ be a convex and $L$-Lipschitz function. Then the smoothing
$\hatf(\boldsymbol{x})=\mathbb{E}_{\uu\sim B_{2}^n}f\left(\boldsymbol{x}+\ndel \uu\right)$
satisfies the following properties:
(1) $\left|f(\x)-\hatf(\x)\right|\leq\ndel L$, 
(2) $\widehat{f}$ is convex and $L$-Lipschitz, 
(3) $\nabla\hatf(\x)=\frac{n}{\ndel}\cdot\mathbb{E}_{\uu\sim S_{2}^n}f(\x+\ndel \uu)\cdot \uu$.
\end{lem}

\medskip
{\bf Tree Based Aggregation.}
An essential ingredient of the algorithm is maintaining partial sums of the gradient estimators witnessed so far. We use a variant of the algorithm from~\citep{dwork2010differential, jain2012differentially}, as implemented in~\citep{agarwal2017price}. We use the algorithm as a black box and only rely on its properties that are stated in Theorem~\ref{thm:treebthm} below. We include a description of the \textsc{TreeBasedAgg} algorithm in Appendix~\ref{sec:tbagg_app} for completeness.

\begin{thm}[\citep{jain2012differentially,agarwal2017price}]
\label{thm:treebthm}
Let $\{\ell_t\}_{t=1}^T$ be a sequence of vectors in $\R^n$, and let $Y_1$ and $Y_2$ be promises such that $\| \ell_t \|_1 \leq Y_1$ and $\| \ell_t \|_2 \leq Y_2$ for all $t$.
Let $\epsilon, \delta > 0$, and $\lambda \geq \frac{Y_1 \log T}{\eps}$ and $\sigma \geq  \frac{Y_2}{\eps} \log T \log \frac{\log T}{\delta}$.

There is an algorithm, \textsc{TreeBasedAgg}, that first outputs $\widehat{L}_0$ and then iteratively takes $\ell_t$ as input and returns an approximate partial sum $\widehat{L}_t$ for $1\leq t \leq T$. The algorithm can be specified with a noise distribution $\mathcal{P}$ over $\R^n$ so that the output sequence $\{\widehat{L}_t\}_{t=1}^T$ satisfies 
$\widehat{L}_t = \sum_{s=1}^t \ell_s + \sum_{r=1}^{\lceil \log T \rceil}Z_r$, where $Z_r \sim \mathcal{P}$, and furthermore:
\begin{itemize}[leftmargin=*,noitemsep,topsep=0pt]
\item when $\mathcal{P}$ is coordinate-wise $Lap(0, \lambda)$, the sequence $\{\widehat{L}_t\}_{t=1}^T$ is $(\epsilon, 0)$-differentially private.
\item when $\mathcal{P}$ is coordinate-wise $\mathcal{N}(0,\sigma^2)$, the sequence $\{\widehat{L}_t\}_{t=1}^T$ is $(\epsilon, \delta)$-differentially private.
\end{itemize}
\end{thm}

\section{Algorithm}\label{sec:mainalg}

The algorithm is described in Algorithm~\ref{alg:bandit}. It builds on the work of Garber and Kretzu \citep{garber2019improved} and uses the smoothing (Lemma~\ref{lem:smooth}) and tree aggregation (Theorem~\ref{thm:treebthm}) routines designed in previous work (see Section~\ref{sec:prelim}). The algorithm follows the structure of an online mirror descent algorithm. 
It performs a sequence of iterations, and in each iteration it makes a guess $x_t$ based on the previous outcomes.
The iterations are divided into $\Tr$ batches, each of size $\Tb$ (thus, $T=\Tr\cdot \Tb$). Each batch $R$ is treated as a round for online mirror descent with a twist: in parallel, we compute the best regularized response $\xtil_R$ for the revealed outcomes in the first $R-1$ batches (line~$14$) and use the previously computed $\xtil_{R-1}$ for all iterations in batch $R$ (lines $6$ to $12$). Three notices are in order:
\begin{itemize}[leftmargin=*,noitemsep,topsep=0pt]
    \item Computing the best regularized response to previous outcomes requires maintaining the sum of gradients in previous rounds. The tree-based aggregation method (Theorem~\ref{thm:treebthm}) is used to maintain these sums accurately while preserving privacy (line~$13$).
    \item In each iteration of a batch, the algorithm only has access to the function value, not the gradient, so we use the smoothing technique (Lemma~\ref{lem:smooth}): the function value at a perturbation of $\xtil_{R-1}$ is used to obtain a stochastic estimate of the gradient of a smoothed proxy of the objective function. Thus, each iteration in the same batch uses a different perturbation of the same response $\xtil_{R-1}$.
    \item We only compute an approximation of the best regularized response, using the conditional gradient method in line 14. The precision to which this is computed is chosen in such a way that the number of iterations required by conditional gradient matches the number of iterations in a batch, so that we can charge each call to the linear optimization oracle over $\mathcal{D}$ to one iteration of the bandit algorithm.
\end{itemize}

\setlength{\columnsep}{2.5cm}
\setlength{\multicolsep}{0.0pt}

\begin{algorithm*}[h]
   \caption{\textsc{PrivateBandit}$(T, \Dist, D)$}
   \label{alg:bandit}
\begin{algorithmic}[1]
   \INPUT time horizon $T$, symmetric noise distribution $\Dist$, diameter of domain $D$.
   \STATE $\Tr = T^{1/2}, \Tb = \frac{T}{\Tr}, \eta = \frac{D}{T^{3/4} n^{1/2} L}$, $\ndel = \frac{Dn^{1/2}}{T^{1/4}}$.
   \STATE Initialize \textsc{TreeBasedAgg} for a sequence of length $\Tr$ and noise $\Dist$.
   \FOR{$R=1$ {\bfseries to} $\Tr$}
	   \STATE \textbf{execute in parallel:}
	   \begin{multicols}{2}
	    \STATE $\gtil_R = 0$\label{line:batch-start}
   		\FOR{$r=1$ {\bfseries to} $\Tb$}
   			\STATE $t = (R-1)\Tb+r$
			\STATE Sample $\uu_t \sim S_2(1)$
			\STATE $\x_{t}=\xtil_{R-1} \textcolor{blue}{ + \ndel \uu_t}$\label{line:answer}
			\STATE Query $F_t = \frac{n}{\ndel}f_t(\xtil_{R-1} +\ndel \uu_t)$
			\STATE $\gtil_R = \gtil_R + F_t \cdot \uu_t$
		\ENDFOR\label{line:batch-end}
		\STATE \COMMENT{Update the partial sum of noisy gradients.} $\stil_{R} = \textsc{TreeBasedAgg}\left(\gtil_R, R\right)$.\label{line:tree}\vfill\columnbreak
		\STATE Solve via conditional gradient
\[
\min_{\x\in\mathcal{D}}\frac{1}{2}\left\Vert \x\right\Vert _{2}^{2}-\left\langle \eta \stil_{R-1},\x\right\rangle 
\]
to precision $\epscg = D/\Tb^{1/2}$.
Let $\xtil_R$ be the output.\label{line:cg}
	    \end{multicols}
   \ENDFOR
\end{algorithmic}
\end{algorithm*}

The algorithm needs to be specified with a noise distribution $\mathcal{P}$ over $\R^n$, which we use for privatizing the partial sums in order to strike the right tradeoff between privacy and regret. To obtain an $(\epsilon,0)$-private algorithm, we set $\mathcal{P}$ to be coordinate-wise Laplace noise $Lap(0,\lambda)$. To obtain an $(\epsilon,\delta)$-private algorithm, we set $\mathcal{P}$ to be coordinate-wise Gaussian noise $\mathcal{N}(0,\sigma^2)$. The precise choice for the parameters $\lambda$ and $\sigma^2$ are established in Lemmas~\ref{lem:reglap} and \ref{lem:reggauss}. 

We analyze the regret and privacy guarantees of the algorithm in Sections~\ref{sec:regret} and \ref{sec:privacy}, respectively. 
 
While our algorithm roughly follows the line of the one from~\citep{garber2019improved}, the version employed here is a slight simplification and generalization of it, since in particular we do not require any specific stopping conditions and case analysis for solving the inner conditional gradient routine, and we can extend it to more general geometries defined by the mirror map. Also, the $\textsc{NoisyOCO}$ framework allows us to handle the noise introduced by the differentially private mechanisms without making any further changes to the algorithm or its analysis. This framework may be of further interest for designing differentially private optimization methods.

In the following section we analyze the regret of Algorithm~\ref{alg:bandit}, for which we prove the following regret bound.
\begin{lem}\label{lem:finalregret}
Let $\mu = \E_{X\sim\Dist}\|X\|$, and let $D$ be the diameter of the domain $\mathcal{D}$, $n$ the ambient dimension, and $L$ an upper bound on the Lipschitz constant of the loss functions $f_t$. Then the algorithm \textsc{PrivateBandit} obtains a regret of
$$O\bigg(
T^{3/4} n^{1/2} LD + T^{1/4} D \mu \log T \bigg) \,.$$
\end{lem}

\section{Noisy Mirror Descent Framework and Regret Analysis}
In this section, we sketch the regret analysis for Algorithm~\ref{alg:bandit}. We derive the algorithm's regret guarantee via the \textsc{NoisyOCO} framework --- a meta-algorithm for online convex optimization with noise --- that we describe and analyze in Section~\ref{sec:noisyoco}. In Section~\ref{sec:regret}, we show that Algorithm~\ref{alg:bandit} is an instantiation of this meta-algorithm and we derive its regret guarantees from the guarantee for \textsc{NoisyOCO}.

\subsection{Noisy Mirror Descent Framework}
\label{sec:noisyoco}

Here we describe and analyze the \textsc{NoisyOCO} algorithm (Algorithm~\ref{alg:noisyoco}) for online convex optimization with noise. We assume that we perform online convex optimization over a convex domain $\D$ endowed with a strongly convex mirror map $\omega : \D \rightarrow \R$ such that $\max_{\x\in\D} \omega(\x) \leq D_{\omega}^2$. We also assume $(\kappa, \gamma)$-noisy gradient access, defined as follows:
\begin{itemize}[leftmargin=*,noitemsep,topsep=0pt]
\item a noisy gradient oracle for $f_t$; 
given $\x$, it returns  a randomized $\gtil = \textsc{NoisyGrad}(f_t, \x)$ such that 
$\E \gtil = \nabla f_t(\x)$, and  $\E \| \gtil \|^2 \leq \kappa^2$,
\item a noisy gradient oracle for $\omega^*$; given $\g$, 
it returns a randomized $\xtil = \textsc{NoisyMap}(\g)$ 
such that 
$\E \|\nabla \omega^*(\g) - \xtil \| \leq \gamma$.
\end{itemize}

\begin{algorithm}[h]
   \caption{\textsc{NoisyOCO}$(T)$}
   \label{alg:noisyoco}
\begin{algorithmic}[1]
\INPUT time horizon $T$.
\STATE $\eta = D_{\omega}^{1/2} / (\kappa T^{1/2}), \s_0 = 0$
\FOR{$t=1$ {\bfseries to} $T$}
	\STATE $\xtil_t = \textsc{NoisyMap}(-\eta \s_{t-1})$
	\STATE {\bfseries output} $\xtil_t$ and 
	       {\bfseries query} $\gtil_t = \textsc{NoisyGrad}(f_t, \xtil_t)$
	\STATE $\s_t = \s_t + \gtil_t$
\ENDFOR
\end{algorithmic}
\end{algorithm}

Under these conditions we can derive the following regret guarantee.

\begin{lem}\label{lem:noisyoco}
Given an instance of online convex optimization with $(\kappa, \gamma)$-noisy gradient access, the algorithm \textsc{NoisyOCO}
obtains an expected regret of $$\Reg_T = O\left(T^{1/2} \kappa D_{\omega} + T \kappa \gamma \right)\,.$$
\end{lem}
We give the full proof in Appendix~\ref{sec:pf_noisy_oco}.

\subsection{Regret Analysis}
\label{sec:regret}
The regret analysis is based on the guarantee for $\textsc{NoisyOCO}$ from Lemma~\ref{lem:noisyoco}. It follows from mapping the steps in Algorithm~\ref{alg:bandit} to the framework from \textsc{NoisyOCO}, and bounding the parameters involved. In order to do so, we explain how the \textsc{NoisyGrad} and \textsc{NoisyMap} routines are implemented by Algorithm~\ref{alg:bandit}. We then proceed to bound the $\kappa$ and $\gamma$ parameters corresponding to this specific instantiation, which will yield the desired result. Here we describe the steps required for analysis. We offer detailed proofs in the appendix.

The first step is to reduce the problem to minimizing regret on a family of functions $\{\tilf_R\}_{R=1}^{\Tr}$, where $\tilf_R = \sum_{t=1}^{\Tr} \hatf_{(\Tb-1)\cdot R + t}$. This introduces two sources of error: one from using the smoothed $\hatf$ instead of $f$, and another from using different iterates $x_t = \xtil_{R-1}+\uu_t$ in the same round, even though batching iterations effectively constrains all the iterates in a fixed batch to be equal. These errors are easy to control, and add at most $O(T\ndel L)$ in regret.

For the family of functions $\{\tilf_R\}_{R=1}^{\Tr}$ we implement $\textsc{NoisyGrad}$ as:
\begin{align*}
\gtil_R &:= \textsc{NoisyGrad}(\tilf_R, \x_t) \\
&= \sum_{t = (R-1)\Tb+1}^{R \Tb} \frac{n}{\ndel} f_t (\x_t + \ndel \uu_t) \uu_t\,,
\end{align*}
which is an unbiased estimator for $\nabla \hatf_R(\x_t)$, per Lemma~\ref{lem:smooth}. Thus we bound $\kappa^2$ by showing in 
Lemma 10.2 that
$\E \|\gtil_R\|^2 \leq \Tb \cdot \left(LD n/\ndel \right)^2 + \Tb^2 L^2$.

Furthermore, the output of \textsc{NoisyMap} is implemented in line 14 by running conditional gradient to approximately minimize a quadratic over the feasible domain $\mathcal{D}$.
The error in the noisy map implementation comes from (1) only approximately minimizing the quadratic, (2) using a noisy partial sum of gradient estimators rather than an exact one, and (3) using a stale partial sum approximation $\stil_{R-1}$ for round $R+1$, instead of $\stil_{R}$. We show in 
Corollary 10.5
that the error parameter corresponding to this \textsc{NoisyMap} implementation can be bounded as 
$\gamma \leq
\eta\left(\lceil \log T \rceil \cdot \mu + \kappa \right) 
+ \sqrt{20} \frac{D}{\sqrt{ \Tb}}$.

In Appendix~\ref{sec:regretanalysis}, we bound the specific parameters corresponding to these implementations. 
Plugging these with bounds inside Lemma~\ref{lem:noisyoco} yields the proof of Lemma~\ref{lem:finalregret}, after appropriately balancing the parameters $\Tb, \Tr, \ndel$.  

\section{Privacy Analysis}\label{sec:privacy}
In this section, we instantiate Algorithm~\ref{alg:bandit} with appropriate noise distribution $\mathcal{P}$ in order to derive our $(\epsilon,0)$-private algorithm and our $(\epsilon,\delta)$-private algorithm. As mentioned earlier, we use Laplace noise for $(\epsilon,0)$-privacy and obtain the guarantee in Lemma~\ref{lem:reglap}, and we use Gaussian noise for $(\epsilon,\delta)$-privacy and obtain the guarantee in Lemma~\ref{lem:reggauss}.

First, we describe the proofs for the $(\epsilon, 0)$-privacy regime, where we employ Laplace noise, since they show how this framework allows us to trade regret and privacy.

\begin{lem}[Privacy with Laplace noise]\label{lem:lappriv}
Let $\mathcal{P}$ be coordinate-wise $Lap(0, T^{1/2} nL \log T /\epsilon)$. The algorithm \textsc{PrivateBandit}$(T, \mathcal{P})$ is $(\epsilon, 0)$-differentially private.
\end{lem}
\begin{proof}
First we bound the $\ell_1$ norm of the vectors whose partial sums are maintained by the tree based aggregation method in Algorithm~\ref{alg:bandit}.
Since each vector contributing to that partial sum is obtained by adding up $\Tb = T^{1/2}$ vectors, each of which is a unit $\ell_2$ vector scaled by a constant that is absolutely bounded by $M = LD \frac{n}{\ndel} = T^{1/4}n^{1/2} L$, we naively bound the $\ell_1$ norm of each of them by $\Tb \cdot n^{1/2} \cdot M \leq T^{1/2} n L$.

Therefore, by Theorem~\ref{thm:treebthm}, releasing $\Tr = T^{1/2}$ such partial sums causes a loss of privacy of at most $\epsilon$ whenever 
$$\lambda \geq \frac{ T^{1/2} n L \log T }{\epsilon}\,.$$
\end{proof}

Using Lemma~\ref{lem:lappriv} we can now bound the regret of the $(\epsilon,0)$-differentially private algorithm.
\begin{lem}[Regret with Laplace noise]
\label{lem:reglap}
Let $\mathcal{P}$ be coordinate-wise $Lap(0, T^{1/2} nL \log T /\epsilon)$. The algorithm \textsc{PrivateBandit}$(T, \mathcal{P})$ has regret
 $$\reg_T = O\bigg(T^{3/4}n^{1/2}LD + \frac{T^{3/4}n^{3/2} L D \log^2 T}{\epsilon} \bigg)\,.$$
\end{lem}
\begin{proof}
Per Lemma~\ref{lem:finalregret} we only need to upper bound the expected $\ell_2$ norm of an $n$-dimensional vector where each coordinate is independently sampled from $Lap(0, T^{1/2} nL \log T /\epsilon)$.  Indeed, we have
$$\mu = O\bigg(n^{1/2} \cdot  T^{1/2} nL \log T /\epsilon\bigg)\,.$$
Plugging this into the regret guarantee from Lemma~\ref{lem:finalregret} we obtain the desired result.
\end{proof}

We notice that the regret guarantee we achieved has an undesirable dependence in dimension. In the remainder of this section, we show that we can obtain a improved guarantees if we settle for $(\epsilon, \delta)$-differential privacy instead, which we achieve by using Gaussian noise. This is also more properly suited to our setting, since the regret bound we proved depends on $\ell_2$ norms of the injected noise vectors, which is exactly what governs the privacy loss in this case. A novel and precise error analysis in Lemmas~\ref{lem:randvec} and \ref{lem:privgauss} enables us to obtain the same regret bound as when projection is available for $(\epsilon, \delta)$-privacy.

We additionally use the fact that the randomized smoothing technique further constrains the norm of the vectors $\gtil_R$ we employ to update the partial sums, with high probability. In order to do this, we resort to a concentration inequality for random vectors (Lemma~\ref{lem:randvec}) which allows us to obtain an improved guarantee on privacy, at the expense of a slight increase in the $\delta$ parameter. Compared to the $(\epsilon, 0)$ case,  allowing this low probability failure event enables us to save roughly  a factor of $\widetilde{O}(T^{1/4}n^{-1/2})$ in the norm of the vectors we use to update the partial sums via tree based aggregation. In turn, these allow us to use less noise to ensure privacy, and therefore we obtain an improved regret.

In order to do so, we first require a high probability bound on the $\ell_2$ norm of a sum of random vectors, multiplied by an adversarially chosen set of scalars.

\begin{lem}\label{lem:randvec}
 Let $\uu_1, \dotsm \uu_k \sim B_2(1)$ be a set of independent random vectors in $\R^n$. Then, with probability at least $1-\delta$, one has that for any vector $\cc \in \R^k$ such that $\|\cc\| \leq \Delta$:
\begin{align*}
\left\Vert \sum_{i=1}^k \uu_i c_i \right\Vert
\leq 10 \Delta \bigg( \log\frac{ n+k}{\delta}
+ \sqrt{\left(1+\frac{k}{n}\right)\log \frac{n+k}{\delta}}\bigg)
\end{align*}
\end{lem}
\begin{proof}
Consider the family of matrices $\{\Z_i\}_{i=1}^k \in \R^{n\times k}$ where $\Z_i$ has its $i^{th}$ column equal to $\uu_i$ and all the other entries are $0$.
Therefore $\E \Z_k = 0$ and $\|\Z_k\| \leq 1$.
Furthermore, by definition $\Z_i \Z_i^\top = \uu_i \uu_i^{\top}$ and $\Z_i^\top \Z_i = \|\uu_i\|^2 \cdot \ones_i \ones_i^\top$.  Therefore $\E \Z_i \Z_i^\top = \mI / n$ and $\E \Z_i \Z_i^{\top} = \ones_i \ones_i^\top$.

So 
\begin{align*}
\sigma^2 &= \max\left\{\left\|\E \sum_{i=1}^k \Z_i \Z_i^\top\right\|, \left\|\E \sum_{i=1}^k \Z_i^\top \Z_i\right\|\right\}  \\
&= \max \left\{ \left\| \frac{k}{n} \cdot \mI \right\|, \left\| \mI \right\|  \right\}
 \leq 1+k/n\,.
\end{align*}
Letting $\Z = \sum_{i=1}^k \Z_i$, and using matrix Bernstein~\citep{tropp2015introduction}, we  see that 
\[
\Pr\left[\left\|\Z \right\| \geq t \right] \leq (n+k)\exp( -t^2 / (2 (1+k/n) + 2t/3 ) )\,.
\]
Hence for $t =10\left(\log  \frac{n+k}{\delta} + \sqrt{\left(1+\frac{k}{n}\right) \log\frac{ n+k}{\delta}}\right)$ one has that $\|\Z\| \leq t$ with probability at least $1-\delta$.

Therefore with probability at least $1-\delta$ we have 
$\left\|\Z\cc \right\| \leq \|\Z\| \|\cc\| \leq 10\Delta \left(\log\frac{ n+k}{\delta}+ \sqrt{\left(1+\frac{k}{n}\right)\log \frac{n+k}{\delta}}\right)$ which implies what we needed.
\end{proof}

Now we can obtain a tighter bound the privacy loss when using Gaussian noise.

\begin{lem}[Privacy with Gaussian noise]\label{lem:privgauss}
Let $\mathcal{P}$ be coordinate-wise $\mathcal{N}(0, \sigma^2)$, where
\begin{align*}
\sigma &= \frac{ T^{1/4}n^{1/2}L \log T \log(T/\delta) }{\epsilon} \\
&\cdot \left( \log\frac{n+T}{\delta}+\sqrt{\left(1+\frac{T^{1/2}}{n}\right)\log \frac{n+T}{\delta}} \right)
\end{align*}
The algorithm \textsc{PrivateBandit}$(T, \mathcal{P})$ is $(\epsilon, \delta)$-differentially private.
\end{lem}
\begin{proof}
Using Lemma~\ref{lem:randvec} and union bound we have that with probability at least $1-\delta_{0} T^{1/2}$ the $\ell_2$ norm of each of the $\Tb$ vectors contributing to the partial sums maintained in the tree based aggregation routine is $M = O\bigg(T^{1/4}n^{1/2}L \left( \log\frac{n+T}{\delta_0}+\sqrt{(1+\frac{\Tb}{n})\log \frac{n+T}{\delta_0}} \right)\bigg)$.

By Theorem~\ref{thm:treebthm} maintaining these partial sums is thus $(\epsilon, \delta_0 T^{1/2} + \delta_1)$-differentially private, where
$\epsilon = \frac{M \log T \log (T/\delta_1) }{\sigma}$. Hence setting $\delta_1 = \delta/2$, $\delta_0 = \delta/(2 T^{1/2})$ and
\begin{align*}
\sigma &= \frac{M \log T \log(T/\delta_1)}{\eps} 
=\frac{ T^{1/4}n^{1/2}L \log T \log(T/\delta) }{\epsilon} \\
&\cdot \left( \log\frac{n+T}{\delta}+\sqrt{\left(1+\frac{T^{1/2}}{n}\right)\log \frac{n+T}{\delta}} \right) 
\end{align*}
yields an $(\epsilon, \delta)$-differentially private algorithm.
\end{proof}

\begin{lem}[Regret with Gaussian noise]\label{lem:reggauss}
Let $\delta =1/(n+T)^{O(1)}$. The algorithm \textsc{PrivateBandit}$(T, \mathcal{N}(0,\sigma^2))$ where
$\sigma$ is chosen according to Lemma~\ref{lem:privgauss} such that the algorithm is $(\epsilon,\delta)$-private
 has regret
\begin{align*}
\Reg_T 
&= O\bigg( T^{3/4} n^{1/2} LD  \\
&+ \frac{T^{1/2}n LD \log^2 T \log(T/\delta)\log((n+T)/\delta)}{\eps} \\
&+\frac{  T^{3/4} n^{1/2} LD \log^2 T \log(T/\delta) \sqrt{\log((n+T)/\delta)} }{\eps} \bigg)\,.
\end{align*}
\end{lem}
\begin{proof}
Per Lemma~\ref{lem:finalregret} we only need to upper bound the expected norm of an $n$-dimensional vector where each coordinate is sampled from $\mathcal{N}(0, \sigma^2)$.  In this case we have $\mu = O(n^{1/2} \sigma)$, so plugging it into the regret guarantee from Lemma~\ref{lem:finalregret} we obtain regret
\begin{align*}
\Reg_T &= O\bigg( T^{3/4} n^{1/2} LD + T^{1/4} n^{1/2} D \sigma \log T  \bigg)
\end{align*}
which implies the result after substituting $\sigma$.
\end{proof}

The proof of Theorem~\ref{thm:maininfo} now follows from combining Lemmas~\ref{lem:lappriv}, ~\ref{lem:reglap}, ~\ref{lem:privgauss} and~\ref{lem:reggauss}.

\section{Discussion and Open Problems}
We saw how one can a derive differentially private algorithm starting from a very basic framework for noisy online convex optimization.
Our analysis builds on advances in both differential privacy and online convex optimization, combines their techniques in non-trivial ways and introduces new ideas. Among others, a novel and precise error analysis in Lemmas~\ref{lem:randvec} and \ref{lem:privgauss} enables us to obtain the same regret bound as when projection is available for $(\epsilon, \delta)$-privacy, in contrast with $(\epsilon, 0)$-privacy. To the best of our knowledge, this is a rare case where such a difference between the two privacy settings arise. We think it is an interesting direction for future work to obtain an analogous improvement even in the $(\epsilon, 0)$-privacy setting.

 It would be interesting to see if this generic method in conjunction with tools from differential privacy can be used to obtain more private learning algorithms. A few outstanding questions remain. Since $\widetilde{O}(T^{3/4})$ is also the best known regret bound in the non-private setting, it would be interesting to improve this result, which may lead to an improved differentially private algorithm. Furthermore, in the non-private setting with projections, obtaining algorithms with lower regret requires a stronger control of the geometry of the domain -- this makes differential privacy more difficult to achieve since even the simplest algorithms with improved regret require solving a linear system of equations, which is much more sensitive to noise than vanilla gradient steps. Developing a more fine-grained understanding of these problems via differential privacy poses an exciting challenge.
 
\section*{Acknowledgments}
AE was supported in part by NSF CAREER grant CCF-1750333, NSF grant CCF-1718342, and NSF grant III-1908510.
HLN was supported in part by NSF CAREER grant CCF-1750716 and NSF grant CCF-1909314.
AV was supported in part by NSF grant CCF-1718342.

\bibliographystyle{plainnat}
\bibliography{main.bib}

\newpage
\appendix
\section{Tree Based Aggregation}\label{sec:tbagg_app}
For completeness, we describe the tree based aggregation routine here. 
The privacy guarantees of the tree based aggregation routine is formally stated in Theorem~\ref{thm:treebthm} from Section~\ref{sec:prelim}. The theorem follows from~\citep{jain2012differentially} (Theorem 9) and is used in~\citep{agarwal2017price} (Theorem 3.1 and Lemma 3.3).

\begin{algorithm}[h]
   \caption{\textsc{TreeBasedAgg-Initialization}$(T, \mathcal{P})$}
   \label{alg:treeb-init}
\begin{algorithmic}[1]
\INPUT time horizon $T$, noise distribution $\mathcal{P}$.
\STATE Create an empty balanced binary tree $B$ with $T$ leaves. The leaves are labeled with binary representations of $0$ to $T-1$ with $\lceil \log T \rceil$ bits including leading zeroes. The parent of $z\circ 0$ and $z\circ 1$ is labeled $z$. Initialize $B_i \gets b_i$ where $b_i\sim \mathcal{P}$.
\STATE Sample $n_0^1,\dots,n_0^{\lceil \log T\rceil}$ independently from $\Dist$.
\OUTPUT $\sum_{i=1}^{\lceil \log T\rceil} n_0^i$.
\end{algorithmic}
\end{algorithm}

\begin{algorithm}[h]
   \caption{\textsc{TreeBasedAgg}$(\ell_t, t)$}
   \label{alg:treeb}
\begin{algorithmic}[1]
\INPUT loss vector $\ell_t$, round $t$.
\STATE $\widetilde{L}'_t, count \gets \textsc{PrivateSum}(\ell_t, t)$.
\STATE Define $r_t = \lceil \log T\rceil - count$.
\STATE Sample $n_t^1,\dots,n_t^{r_t}$ by sampling each coordinate independently from $\Dist$.
\STATE $\widetilde{L}_t \gets \widetilde{L}_t' + \sum_{i=1}^{r_t} n_t^i$.
\OUTPUT $\widetilde{L}_t$.
\end{algorithmic}
\end{algorithm}

\begin{algorithm}[h]
   \caption{\textsc{PrivateSum}$(\ell_t, t)$}
   \label{alg:privatesum}
\begin{algorithmic}[1]
\INPUT Data vector $\ell_t$, round $t$.
\STATE $s_t \gets $ the binary representation of $t-1$ of length $\lceil \log T\rceil$ bits including leading zeroes.
\STATE For each ancestor $B_a$ of $B_{s_t}$, update $B_a \gets B_a + \ell_t$ \COMMENT{Update the data structure with $\ell_t$}
\STATE Let $s_{t'}$ be the binary representation of $t$ with $\lceil \log T\rceil$ bits including leading zeroes. Let $S_t$ be the set of nodes such that both their parent and right sibling are ancestors of $B_{s_{t'}}$ (if $t=T$ then set $S_t$ to have just the root).
\OUTPUT $(\sum_{i \in S_t} B_i, |S_t|)$
\end{algorithmic}
\end{algorithm}

\section{\textsc{NoisyOCO} Framework Analysis}\label{sec:pf_noisy_oco}
\begin{proof}[Proof of Lemma~\ref{lem:noisyoco}]
We analyze regret by tracking $\omega^*(-\eta \s_t)$ as potential function.
Since $\omega$ is strongly-convex we have that $\omega^*$ is $1$-smooth. Hence, expanding the sum using the smoothness property 
we can write
\begin{align*}
\sum_{t=1}^T \langle -\eta \gtil_t, \x^*  \rangle - \omega(\x^*) \leq 
\omega^*(-\eta \s_T)  
\leq  
\omega^*(0) + \sum_{t=1}^T \left( \langle \nabla\omega^*(-\eta \s_{t-1}), -\eta \gtil_t \rangle + \frac{\eta^2}{2} \|\gtil_t\|^2\right)\,.
\end{align*}
The first inequality follows from the Fenchel-Young inequality $\langle \x, \y\rangle \leq \omega(\x) + \omega^*(\y)$, which is implied by (\ref{eq:fenchelyoung}). Rearranging the terms from the left hand side of the first inequality and this one the right hand side of the second inequality, this gives
\begin{equation}\label{eq:a1}
\sum_{t=1}^T  \langle \gtil_t,\nabla\omega^*(-\eta \s_{t-1}) -\x^* \rangle
\leq \frac{\omega^*(0) + \omega(\x^*)}{\eta} + \frac{\eta}{2} \sum_{t=1}^T \|\gtil_t\|^2\,.
\end{equation}
Therefore, taking the expectation:
\begin{align*}
\E \Reg_T &= \E \sum_{t=1}^T f_t(\xtil_t) - f_t(\x^*) 
\overset{(1)}{\leq }
\E \sum_{t=1}^T  \langle \nabla f_t(\xtil_t),\xtil_t -\x^* \rangle \\
&\overset{(2)}{\leq} \frac{\omega^*(0) + \omega(\x^*)}{\eta} + \frac{\eta}{2} \sum_{t=1}^T \E \|\gtil_t\|^2 
+ \E \sum_{t=1}^T \left(
\langle \nabla f_t(\xtil_t) - \gtil_t, \xtil_t - \x^* \rangle 
+ \langle \gtil_t,  \xtil_t - \nabla\omega^*(-\eta \s_{t-1})  \rangle 
\right) 
\\
&\overset{(3)}{\leq} \frac{\omega^*(0) + \omega(\x^*)}{\eta} + \frac{\eta}{2} T \kappa^2 + T \kappa \gamma\\
&\overset{(4)}{=}O(T^{1/2}  \cdot \kappa D_{\omega} + T \cdot \kappa \gamma)\,.
\end{align*}
In the chain of inequalities above, $(1)$ follows from convexity, $(2)$ from applying the bound from (\ref{eq:a1}), $(3)$ from the $(\kappa,\gamma)$-noisy gradient property, and $(4)$ from substituting $\eta = 2\frac{D_\omega}{\kappa T^{1/2}}$.
\end{proof}

\section{Regret Analysis}\label{sec:regretanalysis}
We devote this section to proving Lemma~\ref{lem:finalregret}, where we analyze the regret of Algorithm~\ref{alg:bandit}. We do so by showing that the algorithm is an instantiation of $\textsc{NoisyOCO}$ and applying Lemma~\ref{lem:noisyoco}. The following lemma states the regret guarantee that we obtain for Algorithm~\ref{alg:bandit}.

Crucially, we analyze regret for the smoothed functions $\{\hatf_t\}_{t=1}^T$. The regret thus obtained $\hatreg_T$ will imply a regret of 
\begin{equation}\label{eq:regreg}
 \reg_T \leq \hatreg_T + T\cdot \ndel L 
\end{equation}
 for the original problem.
In order to apply the result from Lemma~\ref{lem:noisyoco}, we need to specify how Algorithm~\ref{alg:bandit} fits the generic framework described in \textsc{NoisyOCO}. More precisely, we highlight how Algorithm~\ref{alg:bandit} implements \textsc{NoisyMap} and \textsc{NoisyGrad}, then bound the errors they introduce.
We need to carefully account for the sources of noise, and control their magnitude.

First, due to batching, Algorithm~\ref{alg:bandit} attempts to obtain a small regret for the family of functions $\{\tilf_R\}_{R=1}^{\Tr}$, where
\begin{equation}\label{eq:tilfdef}
\tilf_R = \sum_{t=1}^{\Tr} \hatf_{(\Tb-1)\cdot R + t}
\end{equation}
Essentially, this forces the iterates returned for all the $\hatf$'s in a fixed batch to be equal. This is obviously not true, since Algorithm~\ref{alg:bandit} returns on line 9 different actions $\x_t$ in each iteration of a given round.
We can, however, for the purpose of the analysis, bound the regret of the algorithm which would return the same $\x_t = \xtil_{R-1}$ in each iteration of round $R$. The difference in regret between this fictitious algorithm and Algorithm~\ref{alg:bandit} can easily be bounded by using the Lipschitz property of $\hatf_t$. Since
\begin{align*}
\left\vert \hatf_t(\xtil_{R-1})-\hatf_t(\x_t)\right\vert 
&=
\left\vert \hatf_t(\xtil_{R-1})-\hatf_t(\xtil_{R-1} + \ndel \uu_t)\right\vert \\
&\leq
L \cdot \ndel \Vert \uu_t \Vert = L\cdot \ndel\,,
\end{align*}
where we used the fact that $\hatf_t$ is $L$-Lipschitz, according to Lemma~\ref{lem:smooth}. This means that compared to the regret bound we will further prove, the regret we actually pay for is bounded by 
\begin{equation}\label{eq:extrabandit}
 T\cdot \ndel L\,.
\end{equation}

At this point we can finally describe how Algorithm~\ref{alg:bandit} implements \textsc{NoisyGrad} and \textsc{NoisyMap}, which will enable us to derive the final regret bound via Lemma~\ref{lem:noisyoco}, to which we will add the extra terms from Equations~\ref{eq:regreg} and~\ref{eq:extrabandit}.

The implementation of $\textsc{NoisyGrad}$ corresponding to Algorithm~\ref{alg:bandit} is given by:
\begin{align*}
\textsc{NoisyGrad}(\tilf_R, \x_t) = \sum_{t = (R-1)\Tb+1}^{R \Tb} \frac{n}{\ndel} f_t (\x_t + \ndel \uu_t) \uu_t\,,
\end{align*}

We verify that it is an unbiased estimator for $\nabla \tilde f_R$, and bound its norm.
\begin{lem}
The vector returned by \textsc{NoisyGrad}$(\tilf_R, \x_t)$ is an unbiased estimator for $\nabla \tilf_R(\x_t)$.
\end{lem}
\begin{proof}
Per Lemma~\ref{lem:smooth} we have that $\E_{\uu_t\sim S_2(1)} \frac{n}{\ndel}  f_t(\x_t+\ndel \uu_t)\uu_t = \nabla \hatf_t(\x_t)$. Then, combining with Equation~\ref{eq:tilfdef} we obtain the desired result.
\end{proof}

The proof of the following lemma follows the one in~\citep{garber2019improved}.

\begin{lem}\label{lem:noisygradnorm}
The vector $\gtil_R$ returned by \textsc{NoisyGrad}$(\tilf_R, \x_t)$ satisfies 
$\E \|\gtil_R\|^2 \leq \Tb \cdot \left(LD n/\ndel \right)^2 + \Tb^2 L^2$.
\end{lem}
\begin{proof}
Let $I=\{(R-1)\Tb+1,\dots, R\Tb\}$. We write:
\begin{align*}
{\E\left\| \gtil_R \right\|^2}
&= 
\E\left\|
 \sum_{t \in I} F_t \cdot \uu_t
\right\|^2
\\
&= \E\sum_{t\in I} F_t^2 + \sum_{i,j\in I, i\neq j} \E\langle F_i \cdot \uu_i, F_j \cdot \uu_j\rangle \\
&= \E\sum_{t\in I} F_t^2 + \sum_{i,j\in I, i\neq j} \langle \E \left[ F_i \cdot \uu_i \right], \E \left[  F_j \cdot \uu_j \right]\rangle\,,
\end{align*}
For the final identity we used the fact that since $\uu_i$ and $\uu_j$ are independent, we also have that $F_i\cdot \uu_i$ and $F_j \cdot \uu_j$ are independent, combined with the fact that $\E \langle X, Y \rangle = \langle \E X, \E Y\rangle$ for independent random variables $X$ and $Y$.

By definition we have that $\left|F_t\right| \leq \frac{n}{\ndel} \cdot \max_{\x \in \D} f(\x)$. Since $f$ is $L$-Lipschitz and the diameter of $\D$ is at most $D$, we bound the latter quantity by $LD$. Therefore $F_t^2 \leq (LD \cdot n/\ndel)^2$. 

Finally, we use the fact that $\E \left[ F_i \cdot \uu_i \right] = \nabla \hatf_i(\x_t)$. Since $\hatf$ is $L$-Lipschitz by Lemma~\ref{lem:smooth}, we have that, $\left\|\E \left[ F_i \cdot \uu_i \right]\right\| = \left\| \nabla \hatf_i(\x_t) \right\| \leq L$.

Now we can finish proving the upper bound. First we use Cauchy-Schwarz to bound the inner products in the last term via products of norms, which combined with the inequalities previously proved gives:
\begin{align*}
\E \|\gtil_R\|^2 \leq \Tb \cdot \left(LD n/\ndel \right)^2 + (\Tb^2 - \Tb) L^2\,.
\end{align*}
\end{proof}

Next, we analyze the iterate returned by the implementation of $\textsc{NoisyMap}$ corresponding to Algorithm~\ref{alg:bandit}. Specifically, we need to bound the expected distance in norm to the true iterate that ought to be returned \[
\hatx_R = \nabla \omega^*\left(-\eta \sum_{t=1}^{R-1} \gtil_t\right)\,.
\]

In order to do so we measure the expected error introduced by using 
$\nabla\omega^*(-\eta \tils_{R-2})$ instead, where
$\tils_{R-2}$ is the output of \textsc{TreeBasedAgg} for privately releasing the partial sum $\sum_{t=1}^{R-2} \tilg_{t}$.
 Since we only have access to an approximation of $\nabla\omega^*$ computed via the parallel conditional gradient routine, we then need to measure the additional error introduced here.

\begin{lem}\label{lem:omegastar_error}
Let $\mu = \E_{X \sim \Dist} \|X\|$ be the expected norm of a vector sampled from the noise distribution $\Dist$, and let $\kappa^2 \geq \E \|\gtil_R\|^2$ for all $R$. Then
for all rounds $R$, one has that
$$\left\| \hatx_R - \nabla\omega^*\left(-\eta \tils_{R-2} \right) \right\| \leq \eta\left(\lceil \log T \rceil \cdot \mu + \kappa \right)\,.$$
\end{lem}
\begin{proof}
First we notice that per the implementation of tree based aggregation, the output $\tils_{R-2}$ is obtained by adding $\lceil \log R \rceil$ terms sampled from the distribution $\Dist$ to the partial sum $\sum_{t=1}^{R-2} \tilg_t$.

Therefore, by triangle inequality, we have 
\begin{align*}
&\E
\left\| 
\tils_{R-2} - \sum_{t=1}^{R-1} \tilg_t
\right\| 
\leq \E \left\| \tils_{R-2} - \sum_{t=1}^{R-2} \tilg_t \right\| + \E \|\tilg_{R-1}\|
\leq \lceil \log T \rceil \cdot \mu + \E \|\tilg_{R-1}\| 
\leq \lceil \log T \rceil \cdot \mu + 
\kappa\,,
\end{align*}
where the last inequality follows from $\left(\E\|X\|\right)^2 \leq \E \|X\|^2$.

Finally, we use the fact that $\omega$ is $1$-strongly convex by definition, and therefore its Fenchel conjugate $\omega^*$ is $1$-smooth (see~\citep{nesterov2005smooth,kakade2009duality}). Therefore
\begin{align*}
&\left\| \hatx_R - \nabla\omega^*(-\eta \tils_{R-2}) \right\|
=
\left\| \nabla\omega^*\left(-\eta \sum_{t=1}^{R-1} \gtil_t\right) - \nabla\omega^*\left(-\eta \tils_{R-2}\right) \right\|  
\leq \eta \left\|\sum_{t=1}^{R-1} \gtil_t -  \tils_{R-2}\right\|\,.
\end{align*}
Combining with the previous inequality, we obtain the desired bound.
\end{proof}

Finally, since instead of $\hatx_R$ we return an approximation obtained via conditional gradient, we bound the error introduced here.

\begin{lem}\label{lem:condgrad}
Let a vector $\vv$ and let $\x$ be the output produced by running conditional gradient for the objective
$$\min_{\x \in \D} q(\x) := \frac{1}{2} \|\x\|_2^2 - \langle \vv, \x \rangle $$
for $k$ iterations. Then
$$\| \x - \x^* \| \leq \sqrt{20}D/\sqrt{k} \,,$$ where $\x^*$ is the optimizer of the objective.
\end{lem}
\begin{proof}
Standard conditional gradient analysis~\citep{bubeck2015convex} shows that after $k$  iterations we can bound the error
$$q(\x) - q(\x^*) \leq 10 \frac{D^2}{k}\,.$$
Since our quadratic objective is $1$-strongly convex we have that $\frac{1}{2}\|\x-\x^*\|^2 \leq q(\x)-q(\x^*)$. Therefore $\|\x-\x^*\| \leq \sqrt{20} D /\sqrt{k}$.
\end{proof}

Combining Lemma~\ref{lem:omegastar_error} and Lemma~\ref{lem:condgrad} we finally bound the error introduced by \textsc{NoisyMap}.

\begin{cor}\label{cor:noisymaperror}
Let $\mu = \E_{X \sim \Dist} \|X\|$ be the expected norm of a vector sampled from the noise distribution $\Dist$, and let $\kappa^2 \geq \E \|\gtil_R\|^2$ for all $R$. With a budget of $b$ calls to the linear optimization oracle per iteration, one has that the error of the \textsc{NoisyMap} implementation corresponding to Algorithm~\ref{alg:bandit} is bounded by
\begin{align*}
\E \|\hatx_R - \xtil_R\| 
&\leq 
\eta\left(\lceil \log T \rceil \cdot \mu + \kappa \right) 
+ \sqrt{20} \frac{D}{\sqrt{ \Tb}}\,,
\end{align*}
\end{cor}

Combining with Lemma~\ref{lem:noisyoco} we obtain the regret guarantee for Algorithm~\ref{alg:bandit}.

\begin{proof}[Proof of Lemma~\ref{lem:finalregret}]
First we bound $\hatreg_T$. We apply Lemma~\ref{lem:noisyoco} for
$$\kappa = \Tb^{1/2}\left(LD n/\ndel \right) + \Tb \cdot  L$$ and
$$\gamma = \eta\left(\lceil \log T \rceil \cdot \mu + \kappa \right) + \sqrt{20} \frac{D}{\sqrt{ \Tb}}\,.$$
We notice that in this case $D_{\omega}^2 = \max_{\x \in \D} \frac{1}{2} \|\x\|^2 \leq D^2$. Therefore 
together with the fact that from (\ref{eq:regreg}) and (\ref{eq:extrabandit})  we have that $\reg_T \leq \hatreg_T + 2\cdot T\ndel L$
we obtain the upper bound:
\begin{align*}
\reg_T 
&= O\bigg( \kappa \Tr^{1/2} D +\Tr  \kappa \gamma + T\ndel L\bigg)
\,.
\end{align*}

Recalling that 
$$\eta = \frac{D}{\kappa \Tr^{1/2}}$$
we further plug in
\[
\Tb = \Tr = T^{1/2}, \quad \ndel = D \sqrt{n}/ T^{1/4}\,,
\]
which implies the bounds
\begin{align*}
\kappa &= O(T^{1/2} n^{1/2} L)\,,\\
\eta &= O\left(\frac{D}{\kappa T^{1/4}}\right) = O\left(\frac{D}{T^{3/4}n^{1/2}L}\right)\,,\\
\gamma &= O\left(\frac{D}{T^{3/4}n^{1/2}L} \cdot \mu \log T +\frac{D}{T^{1/4}} \right)\,.
\end{align*}
Together with the previous inequalities these further imply that 
\begin{align*}
\Reg_T &= O\bigg(
T^{3/4} n^{1/2} LD + T^{1/4} D \mu \log T  
\bigg)
\,.
\end{align*}
\end{proof}

\section{Removing Assumption on the Domain of $f_t$}\label{sec:domain}
Our analysis relies on the assumption that all the functions $f_t$ are defined over $\R^n$, and hence the smoothing technique from Lemma~\ref{lem:smooth} can be applied to any point $\x\in \D$. In some cases, it is conceivable  that $f_t$ is defined only on $\D$. Hence randomized smoothing can not possibly be used everywhere inside the domain, since the point $\x + \ndel \uu$ might land outside $\D$. Just as in~\citep{garber2019improved} one can mitigate this issue, under the assumption that the domain $\D$ contains a sufficiently large $\ell_2$ ball $r\cdot B_2^n$ in its interior.

Indeed, with this assumption one can consider the modified functions $f'_t(\x) = f_t((1-\ndel/ r)\x)$. Performing randomized smoothing on $f'_t$ involves querying $f'_t(\x +\ndel \uu) = f_t(  (1-\ndel/r)(\x + \ndel\uu) )$, where $\uu \in S_2^n$. 

One can easily see that the point $f_t$ is queried on, $(1-\ndel /r)(\x+\ndel \uu) \in \D$. For this it is sufficient that $\ndel \uu \in (\ndel / r) \D$, which automatically holds since $r\cdot B_2^n \in \D$ and hence $r \cdot \uu \in \D$.

Let us now bound the amount of error introduced by performing bandit convex optimization over the functions $f_t'$ instead of $f_t$. This follows from bounding $\|f_t(\x) - f'_t(\x)\| = \|f_t(\x) - f_t((1-\ndel/r)\x)\| \leq L \cdot \|(\ndel / r) \x \|$, where we used the Lipschitz property of $f_t$. We further bound this with $\ndel \cdot L/r \cdot D$, by using $\|\x\|\leq D$.

Therefore the total amount of extra regret introduced by applying this reduction is $T \cdot \ndel \cdot \frac{L D}{r} \leq T^{3/4} n^{1/2} \cdot \frac{L D^2}{r}$, for the choice of $\ndel$ used in Algorithm~\ref{alg:bandit}. Hence in this setting, the regret bound is not altered by more than a factor of $1/r$.

 \end{document}